\documentclass{article}
\usepackage[T1]{fontenc}
\usepackage{amsmath}
\usepackage{amssymb}
\usepackage{lscape}
\usepackage{amsthm}
\usepackage[left=2.5cm, right=2.5cm, top=3cm]{geometry}
\usepackage{hyperref}
\usepackage{algorithm}
\usepackage{bbm}
\usepackage{booktabs}
\usepackage{algpseudocode}
\usepackage{tikz}
\usepackage{caption}
\usepackage{subcaption}
\usepackage{fancyhdr}
\usepackage[inline]{enumitem}
\usepackage{comment}
\usepackage{nicefrac}
\usepackage{bm}
\usepackage{mathrsfs}
\usepackage{multirow}
\usepackage{graphicx}
\usepackage[utf8]{inputenc}
\usepackage{cancel}
\usepackage{mathtools}
\usepackage{natbib} 

\DeclareMathOperator*{\arginf}{arg\,inf}

\newtheorem{theorem}{Theorem}
\newtheorem{lemma}[theorem]{Lemma}
\newtheorem{proposition}[theorem]{Proposition}
\newtheorem{corollary}{Corollary}[theorem]

\newtheorem{definition}[theorem]{Definition}
\newtheorem{remark}{Remark}

\title{Optimal Transport for $\epsilon$-Contaminated Credal Sets\\{\large To the memory of Sayan Mukherjee}}
\author{Michele Caprio}
\date{{\small University of Manchester, Department of Computer Science\\ Kilburn Building, Oxford Road, Manchester M13 9PL, United Kingdom}}

\begin{document}

\maketitle

\begin{abstract}
    We present generalized versions of Monge’s
and Kantorovich’s optimal transport problems with the probabilities being transported replaced by lower probabilities. We show that, when the lower probabilities are the lower envelopes of $\epsilon$-contaminated sets, then our version of Monge's, and a restricted version of our Kantorovich's problems, coincide with their respective classical versions. We also give sufficient conditions for the existence of our version of Kantorovich's optimal plan, and for the two problems to be equivalent. As a byproduct, we show that for $\epsilon$-contaminations the lower probability versions of Monge's and Kantorovich's optimal transport problems need not coincide. The applications of our results to Machine Learning and Artificial Intelligence are also discussed.
\end{abstract}

\section{Introduction}\label{intro}
The concept of stochasticity is pervasive in modern-day artificial intelligence (AI) and machine learning (ML), allowing to capture the lack of determinism that underpins virtually all interesting applications, ranging from the medical domain \citep{stutz,caprio2025conformalized} to trajectory prediction of ballistic missiles \citep{JI2022458}.

Two objects that are often of interest are a random quantity $\xi_1$, distributed according to a probability measure $P_1$, $\xi_1 \sim P_1$, and a transformation of $\xi_1$ via a function $T$, that we write $\xi_2=T(\xi_1)$, which in turn is distributed according to a probability measure $P_2$, $T(\xi_1)=\xi_2 \sim P_2$. Simple -- but important -- examples of these instances are height and body mass index (BMI) of a population, and a mother's income and her children's (future) income.

From a classical measure-theoretic argument \citep{rudin}, we can obtain $P_2$ as the pushforward measure of $P_1$ via $T$, $P_2(\xi_2 \in B) = P_1(T(\xi_1) \in B) = P_1(\xi_1 \in T^{-1}(B))$, where $B$ is an arbitrary subset of the space that $\xi_2$ take values on. Then, we can write $P_2 \equiv T_{\#} P_1 \coloneqq P_1 \circ T^{-1}$, so that $P_2$ is indeed the pushforward measure $T_{\#} P_1$ of $P_1$ via $T$.

An interesting question we may ask ourselves at this point is whether we can turn the problem around. Given $P_1$ and $P_2$, there are many functions $T$ that push $P_1$ to $P_2$. Is there an ``optimal'' one, that is, one 
that makes the transformation from $P_1$ to $P_2$ as efficient (i.e. less expensive) as possible? This question, which lays at the heart of the field of Optimal Transport (OT), is similar to one that Napoleonic engineers were asked by Napoleon himself. They were tasked to find the cheapest way of transporting iron ore from the mines to the factories \citep{villani}. 

To find such an optimal $T$, in the late 1700s Gaspard Monge suggested the following optimization problem,
\begin{align*}
    \arginf \left\lbrace{\int c(\xi_1,T(\xi_1)) P_1(\text{d}\xi_1) : T_\# P_1=P_2 }\right\rbrace.
\end{align*}
It seeks to find the function $T$ that makes transporting the probability mass encoded in $P_1$ to that encoded in its pushforward via $T$, $P_2=T_\#P_1$, as ``cheap'' as possible. The latter is gauged by considering a cost function $c$ that gives us
the cost of moving one unit of probability mass from $\xi_1$ to $\xi_2=T(\xi_1)$. In other words, $c$ gives us a measure of the efficiency of ``moving probability bits'' from $P_1$ to $P_2=T_\# P_1$. 

Alas, an optimal solution $T$ to this problem may not exist \citep[Section 1.2]{intro_ot}.
Fortunately, though, Leonid Kantorovich came up with an equivalent formulation of the problem that, under mild conditions, is guaranteed to be well-posed. His expression is the following
$$\arginf \left\lbrace{ \int c(\xi_1,\xi_2) \text{d}\alpha(\xi_1,\xi_2) : \alpha \in \Gamma(P_1,P_2) }\right\rbrace,$$
where $\Gamma(P_1,P_2)$ is the set of all joint probability measures whose marginals are $P_1$ and $P_2$. Instead of looking for the most efficient transportation map $T$ from $\xi_1$ to $\xi_2$, it seeks the ``cheapest'' {\em transportation plan} $\alpha$ between the distributions $P_1$ and $P_2$. The relationship between the optimal transportation plan $\alpha$ and the theory of copulas \citep{nelsen}\footnote{Recall that a {\em copula} is a multivariate cumulative distribution function, for which the marginal probability distribution of each variable is Uniform on the interval $[0, 1]$.} was studied e.g. in \citet{chi,liu2023distortedoptimaltransport}.


Another notable Leonid, Wasserstein, used the tools developed by Kantorovich and other optimal transport theory scholars to study a class of probability metrics that bears his name: for $p\geq 1$,
$$W_p(P_1,P_2)\coloneqq \left[ \inf_{\alpha\in \Gamma(P_1,P_2)} \int d(\xi_1,\xi_2)^p \text{d}\alpha(\xi_1,\xi_2) \right]^\frac{1}{p}$$
is the $p$-Wasserstein metric (on the probability space $P_1$ and $P_2$ are defined on), where cost function $c(\xi_1,\xi_2)$ is the $p$-th power of some metric $d$ on the state space where $\xi_1$ and $\xi_2$ are defined on, e.g. some norm $\|\xi_1-\xi_2\|$. In a sense, the Wasserstein metric allows to endow the probability space with a metric derived from the distance defined on the underlying state space. The concept of Wasserstein distance is ubiquitous in AI and ML, spanning fields such as data-driven control \citep{vivian} and uncertainty quantification \citep{yusuf-icml}.

\textbf{Contributions.} In this paper, we ask ourselves:

{\bf Question 1:} What do Monge's and Kantorovich's problems look like, when instead of transporting probability measures, we transport \textit{lower probabilities}?

Lower probabilities are the imprecise counterpart of classical probabilities that allow to describe the ambiguity faced by the scholar around the true data generating processes \citep{walley,intro_ip,decooman}. We give the first (to the best of our knowledge) definitions of Monge's and Kantorovich's problems for lower probabilities, and then we focus our attention on sets of probabilities $\mathcal{M}(\underline{P})$ (called \textit{cores} of a coherent lower probability, a special type of {\em credal sets}) that are completely characterized by their lower envelope $\underline{P}$ (that is a lower probability).\footnote{In general, credal sets are in one-to-one correspondence with lower prevision functionals \citep{decooman}.} This means that the whole set $\mathcal{M}(\underline{P})$ can be reconstructed by simply looking at lower probability $\underline{P}=\inf_{P \in \mathcal{M}(\underline{P})}P$. A pictorial representation of our endeavor is given in Figure \ref{fig1}.

It is necessary to mention that we are not the first to extend the study of OT beyond classical probability theory. \citet{ipmu,Nguyen-ot} do so for belief functions and random sets, and \citet{Rachev_Olkin_1999,gal2019kantorovich,TORRA2023679} do so for capacities and non-additive measures.

{\bf Question 2:} Is there a class of credal sets completely characterized by their lower probabilities (LPs), for which the LP versions of Monge's and Kantorovich's problems coincide with their classical counterparts?

We show that, for the class of $\epsilon$-contaminated credal sets (which we introduce in \eqref{contam-def}), the answer to Question 2 is positive.\footnote{What we mean is that solving the LP version of Monge's and Kantorovich's problems is equivalent to solving their classical versions. The solutions, then, will trivially coincide.}
This is an important result, as it promises to be fraught with fruitful consequences for many possible applications. We also give sufficient conditions (i) for the existence of the lower probability version of Kantorovich's optimal plan, and (ii) for the two problems formulations to coincide. A byproduct of the latter is that, in general, the lower probability versions of Monge's and Kantorovich's problems need not coincide. 

\begin{figure*}[h!]
\centering
\includegraphics[width=.78\textwidth]{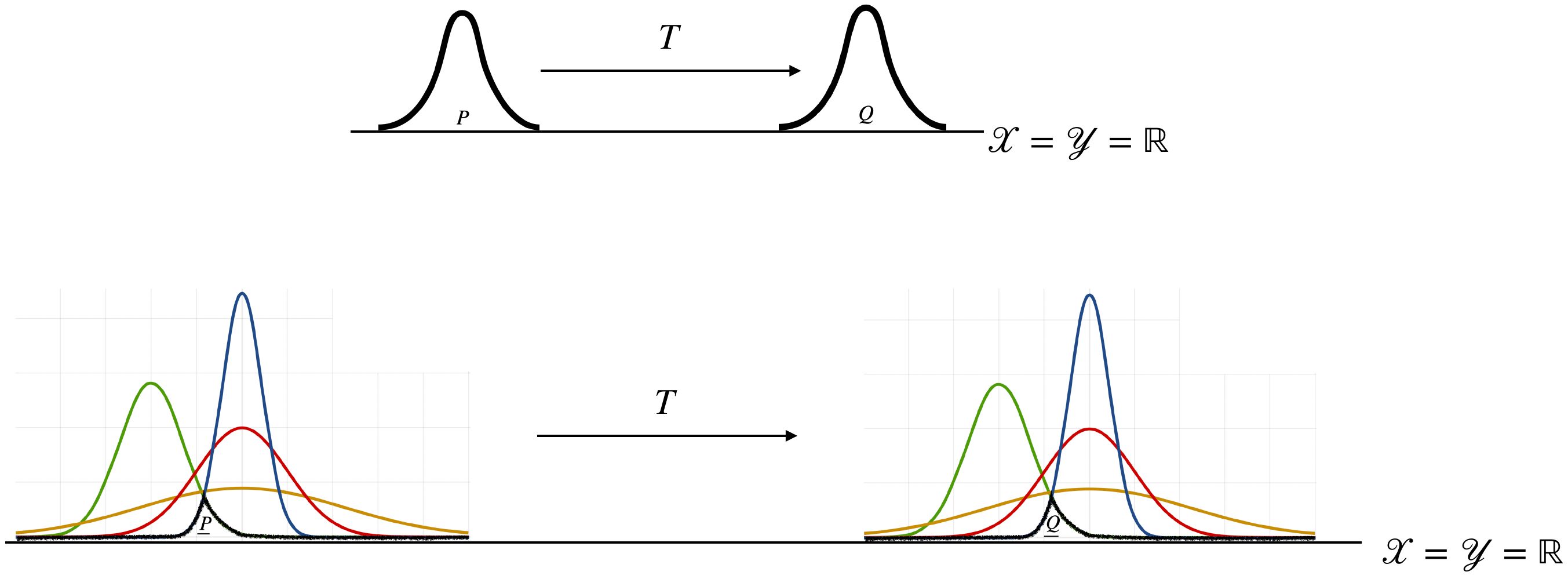}
\caption{Top: the optimal transport map $T$ between two bell-shaped distributions $P$ and $Q$ on $\mathbb{R}$. Bottom: the optimal transport map $T$ between lower probabilities $\underline{P}$ and $\underline{Q}$ (both depicted as black brushstrokes) that completely characterize credal sets $\mathcal{P}$ and $\mathcal{Q}$ of probabilities on $\mathbb{R}$. The colored distributions are elements of the respective credal sets.}\label{fig1}
\end{figure*}

\textbf{Motivation and Related Work.} Besides being interested in these results for their own mathematical beauty, our motivations to study them stem from the field of Imprecise Probabilistic Machine Learning (IPML) \citep{ibnn,dipk,ergo_me,eyke,credal-learning,denoeux,dest1,ibcl,zaffalon}. Credal Machine Learning (CML), a subfield of IPML, devotes itself to developing ML theory and methods working with credal sets. Our findings in this paper can be immediately applied to CML in at least three different contexts. 

First, they can be used to define new uncertainty measures enjoying the axiomatic desiderata in \citet{abellan3,jiro,volume} based on Hausdorff-type distances between sets of probabilities, which would extend the results in \citet{yusuf-icml} to credal sets. To be more specific, the Hausdorff distance based on the Wasserstein metric between two credal sets could be expressed as the Wasserstein distance between their lower probabilities. The latter is intimately related to the Kantorovich's OT problem that we define in this paper.

A second immediate application is robust Hypotheses Testing (HT) \citep{AUGUSTIN2002149,opt_test,gao,mortier,xing,siu-krik,huber-capacities}, where a test statistic based on the optimal transport cost between lower probabilities characterizing credal sets could be used to test whether the true data generating process -- that produces the data accruing to the phenomenon of interest -- belongs to either of the credal sets. In HT notation, $H_0: P^\text{true}$ belongs to $\mathcal{M}(\underline{P}_1)$, versus $H_1: P^\text{true}$ belongs to $\mathcal{M}(\underline{P}_2)$. Let us give a hand-wavy example.  \citet{opt_test} construct a procedure to test $H_0: P^\text{true} \in \mathcal{M}(\underline{P}_1)$ versus $H_1: P^\text{true} \not\in \mathcal{M}(\underline{P}_1)$. To carry out the test, they first split the available sample size into two parts, and use the first part to ``project'' $P^\text{true}$ onto $\mathcal{M}(\underline{P}_1)$; denote by $\hat{P}^\text{true}$ such a projection. Then, with the second part, they test $H_0^\prime: P^\text{true} \in B_\eta(\hat{P}^\text{true})$, where $B_\eta(\hat{P}^\text{true})$ is a ball (with respect to some metric, e.g. the Wasserstein one) of radius $\eta$ centered at $\hat{P}^\text{true}$, versus $H_1^\prime: P^\text{true} \not\in B_\eta(\hat{P}^\text{true})$, and they show that this procedure corresponds to the original test. In the spirit of \citet{ramdas-entropy}, we conjecture that -- if the goal of the researcher is to verify which credal set between $\mathcal{M}(\underline{P}_1)$ and $\mathcal{M}(\underline{P}_2)$ contains $P^\text{true}$ -- this procedure could be extended to one that includes the notion of Wasserstein metric between $\underline{P}_1$ and $\underline{P}_2$ to take into account how far $\mathcal{M}(\underline{P}_1)$ and $\mathcal{M}(\underline{P}_2)$ are from each other, thus making our contributions relevant. 

Finally, our findings can be seminal in starting a new field of inquiry that generalizes ergodic transport theory \citep{lopes,9147930,yifan} to the credal setting. This may also be related to the field of computer vision (and especially the theory of convolutional autoencoders \citep{YU202348}). To see how, we refer the interested reader to  \citet{impr-MSG}, where the ergodicity of (a version of) imprecise Markov processes is related to the behavior of the outputs of a convolutional autoencoder, as the inputs are perturbed.


\textbf{Structure of the Paper.} The paper is arranged as follows. Section \ref{back} gives the necessary background on credal sets. Our results pertaining the lower probability version of Monge's and Kantorovich's problems are presented in Section \ref{main}. Section \ref{concl} concludes our work.

\section{Background}\label{back}

In this section, we introduce the concepts of the core and of the Choquet integral. The reader who is familiar with these notions can skip to Section \ref{main}.

The main tool we work with in this paper is a particular type of a credal set (convex and weak$^\star$-closed set of probabilities \citep{levi2}),\footnote{The weak$^\star$ convergence is the setwise convergence on the fixed $\sigma$-algebra.} that is, what economists and operations researchers call the \textit{core} (of an exact capacity) \citep{cerreia,ergo_me,miranda}.

Given a capacity of interest -- in this paper, it will always be a lower probability $\underline{P}$, i.e. a set function on the $\sigma$-algebra of interest, mapping in $[0,1]$, which is the lower envelope of a weak$^\star$-compact set \citep[Section 2.1.(viii)]{cerreia} -- on a generic measurable space $(\mathcal{X},\mathcal{F})$, the core is defined as
\begin{align}\label{core_first}
    \mathcal{M}^\text{fa}(\underline{P}) \coloneqq \{P \in \Delta^\text{fa}_\mathcal{X}: P(A) \geq \underline{P}(A)\text{, } \forall A \in\mathcal{F}\},
\end{align}
where $\Delta^\text{fa}_\mathcal{X}$ denotes the set of finitely additive probabilities on $(\mathcal{X},\mathcal{F})$.

We focus on cores for two main reasons. First, in general we have that the convex hull of a finite set of finitely additive probabilities on $\mathcal{X}$ is a \textit{proper subset} of the core of the lower probability associated with that set, see e.g. \cite[Example 1]{amarante2} and \cite[Examples 6,7,8]{amarante}. That is, given $\{P_k\}_{k=1}^K \subset \Delta^\text{fa}_\mathcal{X}$, $K < \infty$, we have that $\text{CH}(\{P_k\}_{k=1}^K) \subset \mathcal{M}^\text{fa}(\underline{P})$, where CH denotes the convex hull operator, and $\underline{P}(A)=\inf_{P\in \text{CH}(\{P_k\}_{k=1}^K)}P(A)$, for all $A\in\mathcal{F}$. Hence, focusing on the core gives us more generality.

Second, the core is \textit{uniquely identified} by its lower probability \citep{gong}. To see this, notice that by knowing $\underline{P}$, we can reconstruct the set by simply considering all finitely additive probability measures on $\mathcal{X}$ that set-wise dominate $\underline{P}$. 

Before proceeding to the main results of this paper, we need to introduce the concepts of pushforward lower probability (PLP) and of Choquet integral.

\begin{definition}[Pushforward Lower Probability, PLP]\label{pf}
    Given two measurable spaces $(\mathcal{X},\mathcal{F})$ and $(\mathcal{Y},\mathcal{G})$, a measurable 
    mapping $T: \mathcal{X} \rightarrow \mathcal{Y}$, and a (coherent à la \citet[Section 2.5]{walley}) lower probability $\underline{P}=\inf_{P\in\mathcal{M}^\text{fa}(\underline{P})}P$, the pushforward of $\underline{P}$ is the (set) function $T_\#\underline{P}:\mathcal{G} \rightarrow [0,1]$ such that
    $$T_\#\underline{P}(B)=\underline{P}(T^{-1}(B)), \quad \forall B \in \mathcal{G}.$$
\end{definition}
\begin{lemma}[PLPs are well-defined]\label{well-def}
    The pushforward lower probability $T_\#\underline{P}$ in Definition \ref{pf} is a well-defined lower probability. 
\end{lemma}

Lemma \ref{well-def} -- whose proof relates PLPs to the mathematics of ambiguity \citep{marinacci2}, and is postponed to Appendix \ref{extra-proof} -- entails that $T_\#\underline{P}$ is a coherent lower probability on $\mathcal{G}$, hence a superadditive version of a pushforward probability measure \citep[Section 1.6.4]{walley}. We now introduce Choquet integrals \citep{choquet}, \citep[Section C.2]{decooman}.

\begin{definition}[Choquet Integral]
    Let $(\mathcal{Z},\mathcal{H})$ be a generic measurable space, and $\underline{P}$ be a generic lower probability on $\mathcal{Z}$. For each real-valued function $f$ on $\mathcal{Z}$, we associate the extended real number
    \begin{align}\label{choq_int}
    \begin{split}
        f \mapsto \int_\mathcal{Z} f \text{d}\underline{P} &\coloneqq \int_0^\infty \underline{P}^\star(\{f^+ \geq t\}) \text{d}t \\
        &- \int_0^\infty \left[ \underline{P}^\star(\mathcal{Z}) - \underline{P}^\star(\{f^- \geq t\}^c) \right] \text{d}t
    \end{split}
    \end{align}
    called the {\em Choquet integral} of $f$ with respect to $\underline{P}$, provided that the difference on the right-hand side is well defined. There, $f^+ \coloneqq 0 \vee f$, and $f^- \coloneqq -(0 \wedge f)$. Also, $\underline{P}^\star(A)\coloneqq \sup_{B\subseteq A}\underline{P}(B)$, for all $A \in \mathcal{H}$, is the inner lower probability \citep[Chapter 3.1]{walley},\footnote{We need to work with $\underline{P}^\star$ because $f$ may not be measurable. When it is, $\underline{P}^\star=\underline{P}$.} and the integrals are (improper) Riemannian integrals.\footnote{For a primer on Riemannian integrals, see \citet[Section C.1]{decooman}.}
\end{definition}
When \eqref{choq_int} is well defined, we say that the Choquet integral $\int_\mathcal{Z} f \text{d}\underline{P}$ of $f$ with respect to $\underline{P}$ {\em exists}. We now report \citet[Proposition C.3]{decooman}, which gives an alternative expression of $\int_\mathcal{Z} f \text{d}\underline{P}$, and a sufficient condition for its existence.

\begin{proposition}[Characterizing the Choquet Integral]\label{choq-characteriz}
    Using the same notation as Definition \ref{choq_int}, suppose the Choquet integral $\int_\mathcal{Z} f \text{\em \text{d}}\underline{P}$ of $f$ with respect to $\underline{P}$ exists. Then,
    \begin{align*}
        \int_\mathcal{Z} f \text{\em \text{d}}\underline{P} &= \int_0^\infty \underline{P}^\star(\{f \geq t\}) \text{\em \text{d}}t\\
        &+ \int_{-\infty}^0 \left[ \underline{P}^\star(\{f \geq t\}) - \underline{P}^\star(\mathcal{Z}) \right] \text{\em\text{d}}t.
    \end{align*}
    In addition, if $f$ is bounded or Borel measurable, then it is {\em Choquet integrable} with respect to $\underline{P}$, that is, its Choquet integral $\int_\mathcal{Z} f \text{\em \text{d}}\underline{P}$ exists.
\end{proposition}

\begin{corollary}[A Simplification of the Choquet Integral] \label{simpl}
    Using the same notation as Definition \ref{choq_int}, if $f$ is positive and measurable, then $\int_\mathcal{Z} f \text{\em \text{d}}\underline{P} = \int_0^\infty \underline{P}(\{f \geq t\}) \text{\em \text{d}}t$. If $f$ is also bounded, then the weak inequality can be substituted by a strict one. 
\end{corollary}

\begin{proof}
    The first part of the statement comes from Proposition \ref{choq-characteriz} and \citet[Equation (11)]{marinacci2}. The second part is a consequence of \citet[Proposition 17]{marinacci2}. A similar result to Corollary \ref{simpl} was proven by \citet{Grabisch2016}.
\end{proof}

\section{Main Results}\label{main}

In this section, we answer the two questions that we put forth in the Introduction. First, we give the general definitions of Monge's and Kantorovich's problems for transporting lower probabilities $\underline{P}$ and $\underline{Q}$. Then, we notice how for a special type of cores $\mathcal{M}(\underline{P})$ and $\mathcal{M}(\underline{Q})$, namely those associated with contaminated (countably additive) probabilities $P$ and $Q$, such problems are equivalent to the classical ones, where the transport happens between $P$ and $Q$. This is because $\mathcal{M}(\underline{P})$ and $\mathcal{M}(\underline{Q})$ are completely characterized by the lower probabilities $\underline{P}=(1-\epsilon)P$ and $\underline{Q}=(1-\epsilon)Q$. We show formally how the intuition that the $(1-\epsilon)$ scaling factor does not play a role when solving Monge's and Kantorovich's problems is correct. As a result, the lower probabilities and the classical versions of such problems coincide.

We begin by remarking two notational choices and an assumption that we make in the rest of the paper. We will put $\mathcal{M}^\text{fa}(\underline{P}) \equiv \mathcal{M}(\underline{P})$, and we will call $\Delta^\text{ca}_\mathcal{X}$ the space of countably additive probabilities on $\mathcal{X}$. We will also assume that the cost function $c$ is (Borel) measurable, so that (i) inner lower probability $\underline{P}^\star$ and ``classical'' lower probability $\underline{P}$ coincide, and (ii) the Choquet integrals that we consider exist. 

The special cores that we consider are the so-called {\em $\epsilon$-contaminated credal sets}. That is, given a countably additive probability measure $P$ on $\mathcal{X}$, $P \in \Delta^\text{ca}_\mathcal{X}$, we consider the set
\begin{align}\label{contam-def}
\begin{split}
    \mathcal{P}_\epsilon=\{\Pi \in \Delta^\text{fa}_\mathcal{X} : \text{ } &\Pi(A)=(1-\epsilon) P(A) + \epsilon R(A),\\
    &\forall R \in \Delta^\text{fa}_\mathcal{X}, \forall A \in \mathcal{F}\},
\end{split}
\end{align}
where $\epsilon$ is a parameter in $[0,1]$. 

\begin{lemma}[Properties of $\epsilon$-Contaminated Credal Sets]\label{lemma-prop}
    Let $\mathcal{P}_\epsilon$ be an $\epsilon$-contaminated credal set as in \eqref{contam-def}. Then, $\underline{P}^\prime(A)$ is given by
    \begin{align}\label{first_lp}
\inf_{\Pi\in\mathcal{P}_\epsilon} \Pi(A) =\begin{cases}
        (1-\epsilon) P(A), & \text{for all } A\in\mathcal{F}\setminus \{\mathcal{X}\}\\
        1, & \text{for } A=\mathcal{X}
    \end{cases}
    \end{align}
    and 
    \begin{align}\label{eq_imp}
        \mathcal{P}_\epsilon=\mathcal{M}(\underline{P}^\prime)=\{\Pi \in \Delta^\text{{\em fa}}_\mathcal{X} : \Pi(A) \geq \underline{P}^\prime(A) \text{, } \forall A\in\mathcal{F}\}.
    \end{align}
\end{lemma}

\begin{proof}
    Both these properties were proven in \citet[Example 3]{wasserman} and in \citet[Section 2.9.2]{walley}.
\end{proof}

A remark is in order. The elements of $\mathcal{P}_\epsilon$ must be finitely additive probabilities, and not merely countably additive, because if that were not the case, then $\mathcal{P}_\epsilon$ would not be weak$^\star$-compact, and so it would not be a well-defined core. Let us give an example, borrowed from \citet{walley}. Consider an $\epsilon$-contamination model on the Naturals $\mathbb{N}$, with $\epsilon=1$ (this is the vacuous lower probability that assigns lower probability $0$ to every natural number). The sequence $(\delta_n)_{n\in\mathbb{N}}$ of Dirac measures $\delta_n$ assigning mass $1$ to ${n} \in \mathbb{N}$ is a sequence of countably additive probability measures that belongs to $\mathcal{P}_\epsilon$. But this sequence has no weak$^\star$ converging subsequence to a countably additive probability measure. If it did, it would have to assign probability $0$ to all of the Naturals $n\in\mathbb{N}$. Hence, $\mathcal{P}_\epsilon$ cannot be weak$^\star$-compact in this case.

This is a technicality which does not affect the interpretation of our results, for two main reasons. First, all countably additive probabilities are also finitely additive, that is, $\Delta^\text{ca}_\mathcal{X} \subset \Delta^\text{fa}_\mathcal{X}$. Second, we consider contaminations of a countably additive probability $P \in \Delta^\text{ca}_\mathcal{X}$.\footnote{Of course, in general we may have $P \in \Delta^\text{fa}_\mathcal{X}$.} This is because we want to relate the lower probability versions of Monge's and Kantorovich's OT problems (that we introduce later in this Section) with the classical ones, that are formulated for countably additive probabilities.

In the remainder of the paper, we will work with the (incoherent, according to \citet[Section 2.5]{walley}) lower probability $\underline{P}$ such that
\begin{align}\label{second_lp}
    \underline{P}(A)=(1-\epsilon)P(A), \quad \text{for all } A\in\mathcal{F},
\end{align}
in place of $\underline{P}^\prime$. This is an unnormalized countably additive measure, which is Radon measure if the underlying space is separable. The reason we work with $\underline{P}$ in \eqref{second_lp} is twofold: calculations are easier to carry out, and also the following lemma holds. The interested reader can find a further discussion on this choice in Appendix \ref{app-1}.

\begin{lemma}[A More Convenient Core]\label{equality-cores}
    Pick any countably additive probability measure $P$ on $\mathcal{X}$, any $\epsilon \in [0,1]$, and consider the two lower probabilities $\underline{P}^\prime$ and $\underline{P}$ in \eqref{first_lp} and \eqref{second_lp}, respectively. Let $\mathcal{M}(\underline{P})=\{\Pi \in \Delta^\text{\em fa}_\mathcal{X} : \Pi(A) \geq \underline{P}(A) \text{, } \forall A \in\mathcal{F}\}$. Then, 
    $\mathcal{M}(\underline{P}^\prime)=\mathcal{M}(\underline{P})$. 
\end{lemma}

\begin{proof}
    We begin by noting that $\underline{P}^\prime$ is a well-defined lower probability by \citet[Section 2.1.(viii)]{cerreia}. Now, pick any $\Pi \in \Delta^\text{fa}_\mathcal{X}$. We have that $\Pi(A) \geq \underline{P}(A)$ if and only if $\Pi(A) \geq \underline{P}^\prime(A)$, for all $A \in \mathcal{F}\setminus\{\mathcal{X}\}$. In addition, $\Pi(\mathcal{X})=\underline{P}^\prime(\mathcal{X})=1 > 1-\epsilon = \underline{P}(\mathcal{X})$. In turn, this shows that $\mathcal{M}(\underline{P}^\prime)=\mathcal{M}(\underline{P})$, concluding the proof.
\end{proof}

The intuition behind Lemma \ref{equality-cores} is that $\underline{P}^\prime$ and $\underline{P}$ only disagree (by $\epsilon$ much) on the value to assign to $\mathcal{X}$. But any finitely additive probability measure $\Pi$ assigns probability $1$ to the whole state space $\mathcal{X}$. So, to determine whether $\Pi$ belongs to $\mathcal{M}(\underline{P}^\prime)=\mathcal{M}(\underline{P})$, it is enough to check if $\Pi$ set-wise dominates $\underline{P}^\prime$ and $\underline{P}$ on the events in $\mathcal{F}\setminus \{\mathcal{X}\}$. An immediate consequence of this argument is that we can write $\mathcal{M}(\underline{P}^\prime)=\mathcal{M}(\underline{P})=\{\Pi \in \Delta^\text{fa}_\mathcal{X} : \Pi(A) \geq \underline{P}(A) = \underline{P}^\prime(A) \text{, } \forall A \in \mathcal{F}\setminus \{\mathcal{X}\}\}.$


Now, let $\mathcal{Q}_\epsilon \subset \Delta^\text{fa}_\mathcal{Y}$ be a credal set defined similarly to $\mathcal{P}_\epsilon$, and consider its ``associated'' lower probability $\underline{Q}=(1-\epsilon)Q$, $Q\in \Delta^\text{ca}_\mathcal{Y}$. We are interested in the Optimal Transport (OT) map between $\mathcal{M}(\underline{P})$ and $\mathcal{M}(\underline{Q})$. Because these sets are completely characterized by $\underline{P}$ and $\underline{Q}$, respectively (as we have seen in Section \ref{back}), we focus our attention on such lower probabilities. 

\subsection{Lower Probability Monge's (LPM) Problem}\label{monge-subsec}

We begin our endeavor of finding the OT map by writing a version of Monge's optimal transport problem involving $\underline{P}$ and $\underline{Q}$. It is the following.

\begin{definition}[Lower Probability Monge's OT Problem, LPM] \label{monge}
    Let $c: \mathcal{X}\times \mathcal{Y} \rightarrow \mathbb{R}_+$ be a Borel measurable (cost) function. Given lower probabilities $\underline{P}$ and $\underline{Q}$ on $\mathcal{X}$ and $\mathcal{Y}$, respectively, we want to find the (measurable) 
    optimal transport map $T:\mathcal{X}\rightarrow\mathcal{Y}$ that solves the following optimization problem
    \begin{equation}\label{mon-eq}
        \arginf \left\lbrace{\int_\mathcal{X} c(x,T(x)) \underline{P}(\text{d}x) : T_\#\underline{P}=\underline{Q}}\right\rbrace.
    \end{equation}
\end{definition}

We assume $c$ to be Borel measurable and $T$ to be measurable to ensure that $c(x,T(x))$ is Choquet integrable. Let us also notice that we work with the Choquet integral $\int_\mathcal{X} c(x,T(x)) \underline{P}(\text{d}x)$ because, in the case of $\epsilon$-contaminations, it corresponds to the lower expectation $\inf_{P\in\mathcal{M}(\underline{P})} \int_\mathcal{X} c(x,T(x)) P(\text{d}x)$; this is because $\underline{P}$ is $2$-monotone. In general, though, we have that $\inf_{P\in\mathcal{M}(\underline{P})} \int_\mathcal{X} c(x,T(x)) P(\text{d}x) \leq \int_\mathcal{X} c(x,T(x)) \underline{P}(\text{d}x)$; the Choquet integral (with respect to a lower probability) is an upper bound for the worst-case (i.e. lower) expectation. In the future, we will study how the solution to LPM changes if we consider more general cores (i.e. not necessarily $\epsilon$-contaminations), and if we work with lower expectations in place of Choquet integrals. 

Notice that the OT map $T$ need not exist,\footnote{In the classical (precise) case, an optimal transport map $T$ does not exist when $P$ is a Dirac measure, but $Q$ is not.} so we will need to verify its existence in every application of interest. We now show that, for $\epsilon$-contaminated credal sets $\mathcal{P}_\epsilon$ and $\mathcal{Q}_\epsilon$, LPM is equivalent to the classical Monge's problem of finding the OT map between the contaminated probabilities $P$ and $Q$. Throughout the rest of the paper, we perpetrate an abuse of terminology and refer to $\underline{P}$ and $\underline{Q}$ as ``lower envelopes'', even though they are not normalized at $1$.

\begin{theorem}[LPM Coincides with Classical Monge for $\epsilon$-Contaminated Credal Sets] \label{thm_equiv_monge}
    Suppose $\mathcal{X}$ and $\mathcal{Y}$ are separable, so that the elements of $\Delta^\text{\em ca}_\mathcal{X}$ and $\Delta^\text{\em ca}_\mathcal{Y}$ are Radon measures. If $\underline{P}$ and $\underline{Q}$ are the lower envelopes of the $\epsilon$-contaminations $\mathcal{P}_\epsilon \subseteq \Delta^\text{\em fa}_\mathcal{X}$ and $\mathcal{Q}_\epsilon \subseteq \Delta^\text{\em fa}_\mathcal{Y}$ of $P\in \Delta^\text{\em ca}_\mathcal{X}$ and $Q \in \Delta^\text{\em ca}_\mathcal{Y}$, respectively, then the LPM of Definition \ref{monge} is equivalent to the classical Monge’s OT Problem involving $P$ and $Q$.
\end{theorem}


\begin{proof}
    By Proposition \ref{choq-characteriz}, Corollary \ref{simpl}, and Lemma \ref{equality-cores}, we have that 
\begin{align}
    \int_\mathcal{X} &c(x,T(x)) \underline{P}(\text{d}x) \nonumber \\
    &=\int_0^\infty \underline{P}\left( \{x \in \mathcal{X} : c(x,T(x)) \geq t\} \right) \text{d}t \nonumber \\
    &=\int_0^\infty (1-\epsilon){P}\left( \{x \in \mathcal{X} : c(x,T(x)) \geq t\} \right) \text{d}t \label{trick} \\
    &=(1-\epsilon)\int_0^\infty {P}\left( \{x \in \mathcal{X} : c(x,T(x)) \geq t\} \right) \text{d}t \nonumber \\
    &= (1-\epsilon)\int_\mathcal{X} c(x,T(x)) P(\text{d}x). \nonumber
\end{align}
    In addition,
\begin{align}\label{eq1}
    T_\#\underline{P}(B)= \underline{P}(T^{-1}(B)) =(1-\epsilon)P(T^{-1}(B)), 
\end{align}
for all $B \in \mathcal{G}$, and
\begin{align}\label{eq2}
    \underline{Q}(B)=(1-\epsilon){Q}(B), \quad \forall B \in \mathcal{G}.
\end{align}
Hence, by \eqref{eq1} and \eqref{eq2}, the constraint in \eqref{mon-eq} becomes 
\begin{align*}
    T_\#\underline{P}=\underline{Q} &\iff (1-\epsilon)P\circ T^{-1} = (1-\epsilon){Q}\\
    &\iff P\circ T^{-1} = Q \iff T_\#{P}={Q}.
\end{align*}
In turn, we can rewrite the LPM in equation \eqref{mon-eq} as
\begin{align}
    &\arginf \left\lbrace{\int_\mathcal{X} c(x,T(x)) \underline{P}(\text{d}x) : T_\#\underline{P}=\underline{Q}}\right\rbrace \nonumber \\
    &=\arginf \left\lbrace{(1-\epsilon)\int_\mathcal{X} c(x,T(x)) P(\text{d}x) : T_\#{P}={Q}}\right\rbrace \nonumber \\
    &=\arginf \left\lbrace{\int_\mathcal{X} c(x,T(x)) P(\text{d}x) : T_\#{P}={Q}}\right\rbrace, \label{monge_classical}
\end{align}
which is the classical Monge's OT problem. The last equality holds because $c$ is a positive functional. Notice also that since $\underline{P}$ and $\underline{Q}$ are both countably additive measures normed to $(1-\epsilon)$, the Choquet integrals in this proof coincide with classical Lebesgue-Stieltjes integrals.
\end{proof}

In Theorem \ref{thm_equiv_monge} we assume separability of $\mathcal{X}$ and $\mathcal{Y}$ to further relate the LPM and the classical Monge's OT problems. Indeed, in the latter, separability ensures the existence of Borel measurable selections, which is crucial for defining transport maps. Notice also that for Theorem \ref{thm_equiv_monge} to hold we do not need to implicitly assume that the contaminating parameter $\epsilon$ is the same for both $\mathcal{P}_\epsilon$ and $\mathcal{Q}_\epsilon$. That is, we could consider $\mathcal{P}_\epsilon$ and $\mathcal{Q}_{\epsilon^\prime}$, $\epsilon^\prime \neq \epsilon$. This because, for the equivalences below \eqref{eq2} to work, it must be that $(1-\epsilon)/(1-\epsilon^\prime)=1$, and so $\epsilon^\prime = \epsilon$ must hold.


We now give an example, formulated as a corollary, in which Theorem \ref{thm_equiv_monge} proves useful. 
\begin{corollary}[OT Map for LPM when $\mathcal{X}=\mathcal{Y}=\mathbb{R}$]\label{thm1}
    Let $\mathcal{P}_\epsilon$ and $\mathcal{Q}_\epsilon$ denote the $\epsilon$-contaminations of countably additive probability measures $P$ and $Q$ on $\mathcal{X}=\mathcal{Y}=\mathbb{R}$. Choose cost function $c$ such that $c(x,y)=h(x-y)$, where $h$ is a strictly convex, positive, Borel measurable functional.
    If $P$ and $Q$ have finite $p$-th moment, $p\in [1,\infty)$, and $P$ has no atom, then the unique solution to LPM is $T=F^{-1}_Q \circ F_P,$ where $F_P$ and $F_Q$ are the cdf's of $P$ and $Q$, respectively.
\end{corollary}

\begin{proof}
 \citet{rachev} show that, given our assumptions on  $P$ and $Q$, an optimal transport map $T$ that attains the infimum in \eqref{monge_classical} exists, is unique, and is given by $T=F^{-1}_Q \circ F_P$. By Theorem \ref{thm_equiv_monge}, then, we know that the same OT map attains the infimum in \eqref{mon-eq}. This concludes the proof.
\end{proof}

\subsection{Lower Probability Kantorovich's (LPK) Problem}\label{extension}
Adopting the Kantorovich formulation of the OT problem would strengthen our result, since -- as we shall see in Corollary \ref{cor_1} -- a suitable choice of the cost function $c$ will ensure us that the OT map $T$ exists. In addition, since most existing OT results are expressed as a solution to the classical Kantorovich OT problem, we would be able to immediately use them in the context of $\epsilon$-contaminated credal sets.



The main difficulty coming from studying Kantorovich's version is that its extension to lower probabilities is not as immediate as the one in Definition \ref{monge}. To see this, notice that a lower probability version of Kantorovich's OT problem is the following. 

\begin{definition}[Lower Probability Kantorovich's OT Problem, LPK]\label{kanto}
Let $c: \mathcal{X}\times \mathcal{Y} \rightarrow \mathbb{R}_+$ be a Borel measurable (cost) function. Given lower probabilities $\underline{P}$ and $\underline{Q}$ on $\mathcal{X}$ and $\mathcal{Y}$, respectively, we want to find the joint lower probability $\underline{\alpha}$ (also called the {\em lower optimal transport plan}) on $\mathcal{X}\times\mathcal{Y}$ that solves the following optimization problem
\begin{align}\label{lpk_problem}
\arginf\left\lbrace{\int_{\mathcal{X}\times\mathcal{Y}} c(x,y) \text{d}\underline{\alpha}(x,y) : \underline{\alpha} \in \Gamma (\underline{P},\underline{Q})}\right\rbrace,
\end{align}
where $\Gamma (\underline{P},\underline{Q})$ is the collection of all joint lower probabilities on $\mathcal{X}\times\mathcal{Y}$ whose marginals on $\mathcal{X}$ and $\mathcal{Y}$ are $\underline{P}$ and $\underline{Q}$, respectively.
\end{definition}
Considerations on the choice of a Borel measurable cost function $c$ and of working with a Choquet integral similar to those pointed out below Definition \ref{monge} hold also for Definition \ref{kanto}. In Imprecise Probability theory \citep{intro_ip,decooman,walley} there is not a unique way to perform conditioning \citep{gong,teddy_me,dipk}, so we need to be extra careful when defining $\Gamma (\underline{P},\underline{Q})$ in \eqref{lpk_problem}. In this work, we
consider the joint lower probabilities resulting from {\em geometric conditioning}, 
and write $\Gamma (\underline{P},\underline{Q}) \equiv \Gamma^\text{geom} (\underline{P},\underline{Q})$.
In that case, the conditional lower probability resulting from a joint probability $\underline{G}$ is derived as
\begin{align}\label{geom-cond}
    \underline{G}(A \mid B)=\frac{\underline{G}(A , B)}{\underline{G}_\mathcal{Y}(B)}, \quad \forall A \in\mathcal{F}, \forall B \text{ s.t. } \underline{G}_\mathcal{Y}(B)>0,
\end{align}
where $\underline{G}_\mathcal{Y} \equiv \underline{Q}$ is the marginal lower probability of $\underline{G}$ on $\mathcal{Y}$, and similarly for $\underline{G}_\mathcal{X}\equiv \underline{P}$.
The importance of the choice of geometric conditioning is further discussed in Appendix \ref{condit-rem}. Let us mention here that -- as pointed out by \citet{gong} -- the agent that chooses the geometric rule as a mechanism to update their belief is a pessimist. In fact, the geometric rule endorses a stringent interpretation of what counts as evidence for both the query ($A$) and conditioning ($B$) events, by admitting only evidence that supports its constituents into the lower conditional probability.

\subsection{Restricted Lower Probability Kantorovich's (RLPK) Problem}\label{rlpk-subsect}

If the marginal lower probabilities correspond to the lower envelopes of $\epsilon$-contaminated credal sets, then using joint lower probabilities that can be decomposed as in \eqref{geom-cond} entails that the elements of $\Gamma (\underline{P},\underline{Q})$ are such that, for all $A \in\mathcal{F}$ and all $B\in\mathcal{G}$ such that $\underline{Q}(B)>0$,
$$\underline{G}(A \mid B)=\frac{\underline{G}(A , B)}{\underline{Q}(B)} = \frac{\underline{G}(A , B)}{(1-\epsilon) Q(B)}$$
and similarly, $\underline{G}(B \mid A)=\frac{\underline{G}(B , A)}{\underline{P}(A)} = \frac{\underline{G}(A , B)}{(1-\epsilon) P(A)}$.


We can then consider a restricted version of LPK.


\begin{definition}[Restricted Lower Probability Kantorovich's OT Problem, RLPK]\label{kanto_restr}
Let $c: \mathcal{X}\times \mathcal{Y} \rightarrow \mathbb{R}_+$ be a Borel measurable (cost) function. Given lower probabilities $\underline{P}$ and $\underline{Q}$ on $\mathcal{X}$ and $\mathcal{Y}$, respectively, we want to find the joint lower probability $\underline{\alpha}$ on $\mathcal{X}\times\mathcal{Y}$ that that solves the following optimization problem
\begin{align}\label{rest_kant}    \arginf \left\lbrace{\int_{\mathcal{X}\times\mathcal{Y}} c(x,y) \text{d}\underline{\alpha}(x,y) : \underline{\alpha} \in \Gamma_R (\underline{P},\underline{Q})}\right\rbrace,
\end{align}
where $\Gamma_R (\underline{P},\underline{Q}) \subset \Gamma (\underline{P},\underline{Q})$ is the collection of all joint lower probabilities (i) that can be written as an $\epsilon$-contamination of countably additive joint probability measures $G \in \Delta^\text{ca}_{\mathcal{X}\times\mathcal{Y}}$, and (ii) whose marginals on $\mathcal{X}$ and $\mathcal{Y}$ are $\underline{P}$ and $\underline{Q}$, respectively. 
\end{definition}

Definition \ref{kanto_restr} entails that, if $\underline{P}$ and $\underline{Q}$ are the lower envelopes of the $\epsilon$-contaminations $\mathcal{P}_\epsilon$ and $\mathcal{Q}_\epsilon$, respectively, then an element $\underline{G} \in \Gamma_R (\underline{P},\underline{Q})$ is such that
$$\underline{G}(A \mid B)=\frac{\underline{G}(A , B)}{\underline{Q}(B)} = \frac{(1-\epsilon){G}(A , B)}{(1-\epsilon) Q(B)} = \frac{{G}(A , B)}{Q(B)},$$
and similarly,
$$\underline{G}(B \mid A) = \frac{\underline{G}(A , B)}{\underline{P}(A)} = \frac{(1-\epsilon){G}(A , B)}{(1-\epsilon) P(A)} = \frac{{G}(A , B)}{P(A)}.$$

Notice how working with $\Gamma_R(\underline{P},\underline{Q})$ is reminiscent of the covariate shift condition in the Machine Learning literature \citep[Section 3.3.4.4]{RAITOHARJU202235}. That is, a situation in which there is ambiguity on the marginal distribution of the input features of the model (but not on the conditional distribution of the output, given the input), which may be different (i.e. may have changed) from the one that the model has ``seen'' during training and validation.

We now show that, for $\epsilon$-contaminated credal sets $\mathcal{P}_\epsilon$ and $\mathcal{Q}_\epsilon$, RLPK is equivalent to the classical Kantorovich's OT problem. The result need not hold if either the unrestricted LPK or a different type of conditioning are considered. We will expand on this in Remark \ref{rem-imp}.


\begin{theorem}[RLPK Coincides with Classical Kantorovich for $\epsilon$-Contaminated Credal Sets]\label{thm-equiv}
    Suppose $\mathcal{X}$ and $\mathcal{Y}$ are separable, so that the elements of $\Delta^\text{\em ca}_\mathcal{X}$ and $\Delta^\text{\em ca}_\mathcal{Y}$ are Radon measures. If $\underline{P}$ and $\underline{Q}$ are the lower envelopes of the $\epsilon$-contaminations $\mathcal{P}_\epsilon \subseteq \Delta^\text{\em fa}_\mathcal{X}$ and $\mathcal{Q}_\epsilon \subseteq \Delta^\text{\em fa}_\mathcal{Y}$ of $P\in \Delta^\text{\em ca}_\mathcal{X}$ and $Q \in \Delta^\text{\em ca}_\mathcal{Y}$, respectively, then the RLPK of Definition \ref{kanto_restr} is equivalent to the classical Kantorovich’s OT Problem involving $P$ and $Q$.
\end{theorem}

\begin{proof}
Pick any element $\underline{G}$ of $\Gamma_R (\underline{P},\underline{Q})$. We have that
\begin{align}
    \underline{G}(A , B) &= \underline{G}(A \mid B) \underline{Q}(B) = \frac{G(A , B)}{Q(B)} (1-\epsilon) Q(B) \label{first-eq-imp}\\
    &= \underline{G}(B \mid A) \underline{P}(A) = \frac{G(A , B)}{P(A)} (1-\epsilon) P(A) \label{second-eq-imp}\\
    &= (1-\epsilon) G(A , B). \nonumber
\end{align}

So we can write $\Gamma_R (\underline{P},\underline{Q}) = (1-\epsilon) \Gamma (P,Q) = \{(1-\epsilon) G : G \in \Gamma (P,Q)\}$, where set $\Gamma(P,Q)$ is the the collection of all (countably additive) probability measures on $\mathcal{X}\times\mathcal{Y}$ whose marginals on $\mathcal{X}$ and $\mathcal{Y}$ are ${P}$ and ${Q}$, respectively. This shows that $\Gamma_R (\underline{P},\underline{Q})$ is nonempty if and only if $\Gamma(P,Q)\neq \emptyset$. In addition, it is easy to see that $\Gamma_R (\underline{P},\underline{Q})$ inherits the convexity and the weak$^\star$-compactness from $\Gamma(P,Q)$.\footnote{It is easy to see that $\Gamma(P,Q)$ is convex. In addition, its weak$^\star$-compactness comes from Prokhorov's theorem, following our separability assumptions and the well-known fact that $\Gamma(P,Q)$ is weak$^\star$-closed.} In turn,
\begin{align}
&\arginf\left\lbrace{\int_{\mathcal{X}\times\mathcal{Y}} c(x,y) \text{d}\underline{\alpha}(x,y) : \underline{\alpha} \in \Gamma_R (\underline{P},\underline{Q})}\right\rbrace = \nonumber\\
&\arginf\left\lbrace{(1-\epsilon)\int_{\mathcal{X}\times\mathcal{Y}} c(x,y) \text{d}{\alpha}(x,y) : {\alpha} \in \Gamma (P,Q)}\right\rbrace \label{last-but-one-eq} \\
&= \arginf\left\lbrace{\int_{\mathcal{X}\times\mathcal{Y}} c(x,y) \text{d}{\alpha}(x,y) : {\alpha} \in \Gamma (P,Q)}\right\rbrace, \label{last_eq}
\end{align}
where \eqref{last-but-one-eq} comes from Proposition \ref{choq-characteriz} and our definition of $\Gamma_R (\underline{P},\underline{Q})$, and the last equality comes from $c$ being positive. The fact that \eqref{last_eq} is the classical Kantorovich’s OT Problem \citep{Kantorovich2006OnTT} concludes our proof. Notice also that since $\underline{P}$ and $\underline{Q}$ are both countably additive measures normed to $(1-\epsilon)$, the Choquet integrals in this proof coincide with classical Lebesgue-Stieltjes integrals.
\end{proof}

Notice that for Theorem \ref{thm-equiv} too we do not need to implicitly assume that the contaminating parameter $\epsilon$ is the same for both $\mathcal{P}_\epsilon$ and $\mathcal{Q}_\epsilon$. That is, we could consider $\mathcal{P}_\epsilon$ and $\mathcal{Q}_{\epsilon^\prime}$, $\epsilon^\prime \neq \epsilon$. This because, by \eqref{first-eq-imp}, we have that $\underline{G}(A,B)=(1-\epsilon^\prime) G(A,B)$, and, by \eqref{second-eq-imp}, that $\underline{G}(A,B)=(1-\epsilon) G(A,B)$. But they must be equal to each other, and so $\epsilon=\epsilon^\prime$ must hold. 

We now give sufficient conditions for the minimizer of \eqref{rest_kant} to exist, in the context of $\epsilon$-contaminated credal sets. First, we need to introduce the concept of tightness of $\Gamma_R(\underline{P},\underline{Q})$.

\begin{definition}[Tightness of $\Gamma_R(\underline{P},\underline{Q})$]\label{def-tight}
    Let $(\mathcal{X}\times\mathcal{Y},\tau)$ be a Hausdorff space. Let $\Sigma_{\mathcal{X}\times\mathcal{Y}}$ be a $\sigma$-algebra on $\mathcal{X}\times\mathcal{Y}$ that {\em contains $\tau$}. That is, every $\tau$-open subset of $\mathcal{X}\times\mathcal{Y}$ is measurable, and $\Sigma_{\mathcal{X}\times\mathcal{Y}}$ is at least as fine as the Borel $\sigma$-algebra on $\mathcal{X}\times\mathcal{Y}$. We say that $\Gamma_R(\underline{P},\underline{Q})$ is {\em tight} if, for all $\delta \in (0,1]$, there exists a $\tau$-compact set $K_\delta \in \Sigma_{\mathcal{X}\times\mathcal{Y}}$ such that, for all $\underline{\alpha} \in \Gamma_R(\underline{P},\underline{Q})$, we have that $\underline{\alpha}(K_\delta) > 1-\delta$.
\end{definition}

\begin{lemma}[Necessary and Sufficient Condition for $\Gamma_R(\underline{P},\underline{Q})$ to be Tight]\label{tight-char}
    Let $\mathcal{X},\mathcal{Y}$ be metric spaces, and $\epsilon\in [0,1)$. Then, set $\Gamma_R(\underline{P},\underline{Q})$ is tight if and only if set $\Gamma({P},{Q})$ is tight.
\end{lemma}

\begin{proof}
    Suppose $\Gamma_R(\underline{P},\underline{Q})$ is tight. Given our assumption that $\mathcal{X}$ and $\mathcal{Y}$ are metric spaces, this implies that $\mathcal{X}$ and $\mathcal{Y}$ are separable. Pick any $\delta \in (0,1]$ and any $\underline{\alpha}\in \Gamma_R(\underline{P},\underline{Q})$. Then, by Definition \ref{def-tight}, we have that $\underline{\alpha}(K_\delta) > 1-\delta$. By the proof of Theorem \ref{thm-equiv}, we know that there exists $\alpha \in \Gamma({P},{Q})$ such that $\underline{\alpha}(A)=(1-\epsilon){\alpha}(A)$, for all $A\in\Sigma_{\mathcal{X}\times\mathcal{Y}}$. In turn, this implies that $(1-\epsilon){\alpha}(K_\delta) > 1-\delta \iff {\alpha}(K_\delta) > \frac{1-\delta}{1-\epsilon}$. Now let $\frac{1-\delta}{1-\epsilon} \eqqcolon 1-\gamma$, and put $K_\delta \equiv K_\gamma$. We obtain ${\alpha}(K_\gamma) > 1-\gamma$. But $\delta$ and $\underline{\alpha}$ were chosen arbitrarily, which allows us to conclude that $\Gamma({P},{Q})$ is tight.

    Suppose instead that $\Gamma({P},{Q})$ is tight. As before, given our assumption that $\mathcal{X}$ and $\mathcal{Y}$ are metric spaces, this implies that $\mathcal{X}$ and $\mathcal{Y}$ are separable. Pick any $\delta \in (0,1]$, and any $\alpha \in \Gamma({P},{Q})$. Then, by Definition \ref{def-tight}, we have that ${\alpha}(K_\delta) > 1-\delta$. This holds if and only if $(1-\epsilon){\alpha}(K_\delta) = \underline{\alpha}(K_\delta) > (1-\epsilon)(1-\delta)$, where $\epsilon\in [0,1)$ is the same parameter of Definition \ref{kanto_restr}. Now let $(1-\epsilon)(1-\delta) \eqqcolon 1-\gamma$, and put $K_\delta \equiv K_\gamma$. We obtain $\underline{\alpha}(K_\gamma) > 1-\gamma$. But $\delta$ and ${\alpha}$ were chosen arbitrarily, which allows us to conclude that $\Gamma_R(\underline{P},\underline{Q})$ is tight.
\end{proof}
We are ready for our result.

\begin{corollary}[Existence of OT Plan]\label{cor_1}
    Let $\mathcal{X},\mathcal{Y}$ be metric spaces, and $\epsilon\in [0,1)$. If $\Gamma_R (\underline{P},\underline{Q})$ is tight, and if cost function $c$ in \eqref{rest_kant} is also lower semicontinuous, then a minimizer for \eqref{rest_kant} exists. 
\end{corollary}

\begin{proof}
    Let $\Gamma_R(\underline{P},\underline{Q})$ be tight. Given our assumption that $\mathcal{X}$ and $\mathcal{Y}$ are metric spaces, this implies that $\mathcal{X}$ and $\mathcal{Y}$ are separable. \citet{luigi} show that if $\Gamma(P,Q)$ is tight, and if $c$ is lower semicontinuous, then there is a minimizer for the classical Kantorovich's OT problem. By Lemma \ref{tight-char},
    we know that -- for any $\epsilon\in [0,1)$ -- if $\Gamma_R (\underline{P},\underline{Q})$ is tight, then so is $\Gamma(P,Q)$. The proof follows by the equivalence established in Theorem \ref{thm-equiv}.
\end{proof}

The tightness condition in Corollary \ref{cor_1} is satisfied e.g. when $\mathcal{X}$ and $\mathcal{Y}$ are both Polish spaces.\footnote{Separable, completely metrizable topological spaces} This is an immediate consequence of \citet[Proposition 1.5]{intro_ot}.

\subsection{Equivalence Between LPM And RLPK Problems}\label{equiv-subsec}

We now inspect when do RLPK and LPM coincide, in the context of $\epsilon$-contaminated credal sets.

\begin{theorem}[RLPK is Equivalent to LPM]\label{coincide}
    Suppose $\mathcal{X}$ and $\mathcal{Y}$ are separable, so that the elements of $\Delta^\text{\em ca}_\mathcal{X}$ and $\Delta^\text{\em ca}_\mathcal{Y}$ are Radon measures. Let $\underline{P}$ and $\underline{Q}$ be the lower envelopes of the $\epsilon$-contaminations $\mathcal{P}_\epsilon \subseteq \Delta^\text{\em fa}_\mathcal{X}$ and $\mathcal{Q}_\epsilon \subseteq \Delta^\text{\em fa}_\mathcal{Y}$ of $P\in \Delta^\text{\em ca}_\mathcal{X}$ and $Q \in \Delta^\text{\em ca}_\mathcal{Y}$, respectively. When the minimizer $\underline{\alpha}$ of RLPK is such that $\text{d}\underline{\alpha}(x,y)=\underline{P}(\text{\em \text{d}}x)\delta_{y=T(x)}$, then $T$ is an optimal transport map and RLPK is equivalent to LPM.
\end{theorem}

\begin{proof}
    Given the way we defined $\underline{P}$ and $\underline{Q}$, we have that $\text{d}\underline{\alpha}(x,y)=\underline{P}(\text{d}x)\delta_{y=T(x)} \iff 
        (1-\epsilon) \text{d}{\alpha}(x,y)= (1-\epsilon) {P}(\text{d}x)\delta_{y=T(x)}\iff 
        \text{d}{\alpha}(x,y)= {P}(\text{d}x)\delta_{y=T(x)}$. \citet[Section 1.2]{intro_ot} shows that when the minimizer ${\alpha}$ of Kantorovich's classical OT problem is such that $\text{d}{\alpha}(x,y)={P}(\text{d}x)\delta_{y=T(x)}$,\footnote{Conditions sufficient for such a condition can be found in \citet[Chapter 4]{intro_ot}.} then $T$ is Monge's OT map, and Monge's and Kantorovich's problems are equivalent. The proof, then, follows by Theorems \ref{thm_equiv_monge} and \ref{thm-equiv}.
\end{proof}

We now give an example, formulated as a corollary, in which Theorem \ref{coincide} proves useful.

\begin{corollary}[Multivariate Normal Case]\label{normal-ex}
    Let $\underline{P}$ and $\underline{Q}$ be the lower envelopes of the $\epsilon$-contaminations $\mathcal{P}_\epsilon, \mathcal{Q}_\epsilon \subseteq \Delta^\text{\em fa}_{\mathbb{R}^d}$ of $P=\mathcal{N}_d(0,\Sigma_P)$ and $Q=\mathcal{N}_d(0,\Sigma_Q)$, two (countably additive) multivariate Normals on $\mathcal{X}=\mathcal{Y}=\mathbb{R}^d$, respectively. Select $c(x,y)=|y-Ax|^2/2$, where $A\in \mathbb{R}^{d \times d}$ is invertible. Then, the optimal map that solves LPM is 
    $x \mapsto T(x)$, 
    \begin{align}\label{opt-map-last-ex}
        T(x)=(A^\top)^{-1} \Sigma_P^{-1/2} \left( \Sigma_P^{1/2} A^\top \Sigma_Q A \Sigma_P^{1/2} \right)^{1/2} \Sigma_P^{-1/2} x
    \end{align}
    and the optimal plan that solves RLPK is $\text{d}\underline{\alpha}(x,y)=\underline{P}(\text{\em \text{d}}x)\delta_{y=T(x)}$.
\end{corollary}

\begin{proof}
Notice that $\mathbb{R}^d$ is separable. \citet{gali} shows that if $P=\mathcal{N}_d(0,\Sigma_P)$,  $Q=\mathcal{N}_d(0,\Sigma_Q)$, and $c(x,y)=|y-Ax|^2/2$, then Monge's (classical) OT map is the one in \eqref{opt-map-last-ex}, and also that $\text{d}{\alpha}(x,y)= {P}(\text{d}x)\delta_{y=T(x)}$, so that Monge's and Kantorovich's (classical) problems are equivalent. 

Now, by Theorem \ref{thm_equiv_monge}, we have that if $\underline{P}$ and $\underline{Q}$ are lower envelopes of the $\epsilon$-contaminations $\mathcal{P}_\epsilon, \mathcal{Q}_\epsilon \subseteq \Delta^\text{fa}_{\mathbb{R}^d}$ of $P=\mathcal{N}_d(0,\Sigma_P)$ and $Q=\mathcal{N}_d(0,\Sigma_Q)$, respectively, then LPM coincides with the classical Monge’s OT Problem involving $P$ and $Q$. In turn, this implies that the  map in \eqref{opt-map-last-ex} is also the OT map between $\underline{P}$ and $\underline{Q}$. In addition, by Theorem \ref{thm-equiv}, we know that RLPK coincides with the classical Kantorovich’s OT Problem involving $P$ and $Q$. This implies that the optimal lower coupling $\underline{\alpha}$ is given by $\text{d}\underline{\alpha}(x,y)=\underline{P}(\text{d}x)\delta_{y=T(x)}$.
\end{proof}

\begin{remark}[On the Equivalence of Monge and Kantorovich]\label{rem-imp}
    A consequence of Theorem \ref{coincide} is that, for $\epsilon$-contaminated credal sets, LPM and LPK with $\Gamma(\underline{P},\underline{Q}) \equiv \Gamma^\text{geom}(\underline{P},\underline{Q})$
    \textbf{need not} coincide. To see this, notice that Theorem \ref{thm-equiv} only holds for the \textbf{restricted} LPK (RLPK) in Definition \ref{kanto_restr}.
    Had we not specified that the joint lower probabilities $\underline{G} \in \Gamma_R(\underline{P},\underline{Q})$ are $\epsilon$-contaminations of countably additive joint probabilities $G\in\Delta^\text{ca}_{\mathcal{X}\times\mathcal{Y}}$, then Theorem \ref{thm-equiv} may not have held. Similarly, had we considered generalized Bayes' conditioning, or other conditioning mechanisms for lower probabilities \citep{teddy_me}, Theorem \ref{thm-equiv} may not have held as well. 

    Whether this is a phenomenon pertaining only to $\epsilon$-contaminated credal sets, or a more general one, will be the subject of future studies.
\end{remark}

\section{Conclusion}\label{concl}
The conclusion that we can derive from this work is that Questions 1 and 2 in the Introduction have a positive answer. We can formulate a version of Monge's and Kantorovich's problems for lower probabilities. In addition, we can indeed find one class of credal sets completely characterized by their lower probability (the class of $\epsilon$-contaminations) for which the optimal transport map and plan coincide with the classical cases.
We also inspected when our versions of the two problems coincide, and find out that this need not hold in general. 

With this work, we begin to explore the exciting venue of optimal transport between lower probabilities completely characterizing credal sets. In the future, we plan to further our study of optimal transport between $\epsilon$-contaminations by deriving a Brenier-type theorem \citep{brenier} and a Kantorovich-Rubinstein-type duality result \citep{gali} (which could potentially have a significant impact in economics \citep{rob1,rob2,rob3,rob4}). We also intend to explore the machine learning applications (especially concerning out-of-distribution detection) of our findings, and of distributionally robust optimization \citep{kuhn}. 
Finally, we will extend our focus to
other types of credal sets that are not necessarily completely characterized by lower probabilities, such as finitely generated credal sets (the convex hull of finitely many distributions), to signed measures, and to second-order distributions, that is, distributions over distributions.


\section*{Acknowledgments}
We wish to express our gratitude to Yusuf Sale for insightful discussions, particularly on Lemma \ref{well-def}, Corollary \ref{normal-ex}, and Remark \ref{rem-imp}. We are also grateful to Krikamol Muandet and Siu Lun (Alan) Chau for suggesting the relationship between RLPK and the covariate shift condition in Machine Learning, and to four anonymous reviewers for their insightful comments. 

We would also like to point out how this work was inspired by Daniel Kuhn's \href{https://www.youtube.com/watch?v=rBgehsbawVY\&ab\_channel=ImpreciseProbabilitiesChannelofSIPTA}{seminar} at the Society for Imprecise Probabilities on December 12, 2023. There, he presented some of his and his colleagues' remarkable works on the relationship between optimal transport and distributionally robust optimization, a field studying decision problems under uncertainty framed as zero-sum games against Nature \citep{pflug1,pflug2,kuhn,Blanchet_Kang_Murthy_2019,blanchet,gao2,bahar,liviu}. This immediately led us to think about the possibility of studying optimal transport between lower probabilities characterizing credal sets.  

Finally, as highlighted on the title page, we would like to dedicate the present work to the memory of Sayan Mukherjee. Without his guidance, we would never have had the strength -- and perhaps the hubris -- to start studying imprecise probability theory.



\appendix
\section{On The Choice Of $\underline{P}$}\label{app-1}
In this section, we discuss our choice of the core $\mathcal{M}(\underline{P})$ in Lemma \ref{equality-cores}. We work with the (core of the incoherent, according to \citet[Section 2.5]{walley}) lower probability $\underline{P}$ because it makes it easier to derive our desired results, e.g. the proof of Theorem \ref{thm_equiv_monge}. 

Had we worked with $\underline{P}^\prime$ instead, we would have had (by Corollary \ref{simpl}) 
$$\int_\mathcal{X} c(x,T(x)) \underline{P}^\prime(\text{d}x)=\int_0^\infty \underline{P}^\prime\left( \{x \in \mathcal{X} : c(x,T(x)) \geq t\} \right) \text{d}t.$$
It is not immediate to show that the latter is equal to $(1-\epsilon) \int_\mathcal{X} c(x,T(x)) {P}(\text{d}x)$. To see this, notice that there might be some value $\bar{t} \in \mathbb{R}_+$ for which all $x \in \mathcal{X}$ are such that $c(x,T(x)) \geq \bar{t}$. In that case, $\underline{P}^\prime \left( \{x \in \mathcal{X} : c(x,T(x)) \geq \bar{t}\} \right) = \underline{P}^\prime(\mathcal{X})=1$, and so the ``trick'' that we used in \eqref{trick} does not work anymore. 

To achieve the desired result easily while working with $\underline{P}^\prime$, we would have had to require that the cost function $c$ is bounded. Indeed, suppose that the latter holds, and call $\underline{c} \coloneqq \inf_{x\in\mathcal{X}} c(x,T(x))$ and $\overline{c} \coloneqq \sup_{x\in\mathcal{X}} c(x,T(x))$. Then, we can use \citet[Theorem C.3.(ii).(C.7)]{decooman} to get 
\begin{align*}
    \int_\mathcal{X} &c(x,T(x)) \underline{P}^\prime(\text{d}x)\\
    &= \underline{P}^\prime(\mathcal{X}) \cdot \underline{c} + \int_{\underline{c}}^{\overline{c}} \underline{P}^\prime\left( \{x \in \mathcal{X} : c(x,T(x)) > t\} \right) \text{d}t\\
    &= \underline{c} + (1-\epsilon) \int_{\underline{c}}^{\overline{c}} {P}\left( \{x \in \mathcal{X} : c(x,T(x)) > t\} \right) \text{d}t\\
    &= (1-\epsilon) \int_\mathcal{X} c(x,T(x)) {P}(\text{d}x).
\end{align*}

As we can see, the desired result becomes either harder to prove (if we only ask for $c$ to be Borel measurable),
or it needs an extra assumption (a bounded cost function $c$). A similar argument holds also for the Kantorovich's results.

Since the main goal of the paper is to transport lower probabilities {\em that completely characterize credal sets}, and since by Lemma \ref{equality-cores} we know that $\mathcal{M}(\underline{P})=\mathcal{M}(\underline{P}^\prime)$, we opted for using the incoherent lower probability $\underline{P}$ instead of the coherent one $\underline{P}^\prime$. 

We conclude with a remark. We acknowledge that working with the incoherent lower probability $\underline{P}$ makes it harder to use the techniques that we employ in this work, for models that are more complex than the $\epsilon$-contaminations that we study. How to overcome this shortcoming will be the object of future work.

\section{On The Difference Between Conditioning Methods}\label{condit-rem}
    Let us illustrate the difference that the choice of conditioning rule makes when working with imprecise probabilities. Suppose that, instead of considering geometric, we choose generalized Bayes' conditioning \citep[Section 6.4]{walley}. That is, for a generic credal set $\mathcal{P}$, for all $A\in\mathcal{F}$ and all $B\in\mathcal{G}$ such that $G(B)>0$, 
    $$\underline{G}^\text{GBC}(A \mid B) \coloneqq \inf_{P\in\mathcal{P}} \left[ \frac{G(A, B)}{G(B)} \right].$$
    By \citet[Theorem 6.4.6]{walley}, we have that 
    $$\inf_{P\in\mathcal{P}} \left[ \frac{G(A, B)}{G(B)} \right]=\frac{\underline{G}(A, B)}{\overline{G}(B)},$$
    so 
    \begin{align*}
        \underline{G}^\text{GBC}(A \mid B) &= \frac{\underline{G}(A, B)}{\overline{G}(B)}\\
        &\leq \frac{\underline{G}(A, B)}{\underline{G}(B)} \eqqcolon \underline{G}^\text{geom}(A \mid B),
    \end{align*}
    since $\overline{G}(B) \geq \underline{G}(B)$, for all $B \in\mathcal{G}$. This was also proven in \citet[Lemma 4.3]{gong}. This inequality still holds true even in a simple model like $\epsilon$-contaminated credal sets that we consider in the present work. To see this, notice that -- in the same notation as Lemma \ref{lemma-prop} -- by \citet[Example 3]{wasserman} in an $\epsilon$-contamination model $\mathcal{P}_\epsilon$ we have that $\overline{P}^\prime(A)= (1-\epsilon) P(A) + \epsilon$, for all $A \in\mathcal{F} \setminus \{\emptyset\}$, and $\overline{P}^\prime(\emptyset)=0$. Similarly to what we did in the main body of the paper, we can focus instead the incoherent upper probability $\overline{P}(A)= (1-\epsilon) P(A) + \epsilon$, for all $A \in\mathcal{F}$. Then, by \citet[Section 6.6.2]{walley}, we have that 
    $$\underline{G}^\text{GBC}(A \mid B) = \frac{(1-\epsilon){G}(A, B)}{(1-\epsilon) Q(B) + \epsilon}$$
    and, similarly, that
    $$\underline{G}^\text{GBC}(B \mid A) = \frac{(1-\epsilon){G}(A, B)}{(1-\epsilon) P(A) + \epsilon}.$$
    In turn, this implies that the elements of $\Gamma_R^\text{GBC}(\underline{P},\underline{Q})$ are such that $\underline{G}(A, B) = \underline{G}^\text{GBC}(A \mid B) \overline{Q}(B)=\underline{G}^\text{GBC}(B \mid A) \overline{P}(A)$, for all $A\in\mathcal{F}$ and all $B\in\mathcal{G}$ (having positive upper probability). They are different than the elements of  $\Gamma_R^\text{geom}(\underline{P},\underline{Q})$ that we introduced in Definition \ref{kanto_restr}.

\section{Proof of Lemma \ref{well-def}}\label{extra-proof}

We begin by noting that $T_\#\underline{P}$ is a real-valued set function, and that $T_\#\underline{P}(\emptyset)=\underline{P}(T^{-1}(\emptyset))=\underline{P}(\emptyset)=0$.\footnote{We always (implicitly) assume that $T^{-1}(\emptyset)=\emptyset$.} In turn, $T_\#\underline{P}$ is what \citet[Section 2.1]{marinacci2} call a \textit{game}. In addition, since $\underline{P}(T^{-1}(B)) \in [0,1]$, for all $B\in\mathcal{G}$, we have that the co-domain of $T_\#\underline{P}$ is $[0,1]$, and so $T_\#\underline{P}$ is what \citet[Section 2.1.2]{marinacci2} call a \textit{bounded game}. We can then consider the core of such a bounded game \citet[Section 2.2]{marinacci2}, which is a slight generalization of the core introduced earlier. It is defined as $\mathcal{M}^\text{bc}(T_\#\underline{P})\coloneqq \{\mu \in \text{bc}(\mathcal{G}) : \mu(B) \geq T_\#\underline{P}(B) \text{, } \forall B\in\mathcal{G} \text{, and } \mu(\mathcal{Y})=T_\#\underline{P}(\mathcal{Y})\}$,  where $\text{bc}(\mathcal{G})$ is the vector spaces of all bounded charges (signed, finitely additive measures) on $\mathcal{G}$. 

Notice that $\mathcal{M}^\text{bc}(T_\#\underline{P})$ is nonempty because it contains the pushforward measures of the elements of $\mathcal{M}^\text{fa}(\underline{P})$ in \eqref{core_first} through map $T$. In formulas, $P\in\mathcal{M}^\text{fa}(\underline{P}) \implies T_\# P \in \mathcal{M}^\text{bc}(T_\#\underline{P})$. To see that this is the case, notice that a finitely additive probability is a special case of a bounded charge,
and that $\mathcal{M}^\text{fa}(\underline{P})$ contains all the finitely additive probabilities on $\mathcal{X}$ that set-wise dominate $\underline{P}$. More formally, pick any $\tilde B \in \mathcal{G}$ and any $P\in\mathcal{M}^\text{fa}(\underline{P})$. Let $\tilde A=T^{-1}(\tilde B)$. Then, by \eqref{core_first}, $T_\#P(\tilde B)= P(T^{-1}(\tilde B)) = P(\tilde A) \geq \underline{P}(\tilde A) = \underline{P}(T^{-1}(\tilde B)) = T_\#\underline{P}(\tilde B)$. But then $T_\#P \in \mathcal{M}^\text{bc}(T_\#\underline{P})$, which shows that $\mathcal{M}^\text{bc}(T_\#\underline{P}) \neq \emptyset$. In turn, by \citet[Proposition 3]{marinacci2}, we have that $\mathcal{M}^\text{bc}(T_\#\underline{P})$ is weak$^\star$-compact. 

Now, notice that $\mathcal{M}^\text{fa}(T_\#\underline{P})\coloneqq \{Q \in \Delta^\text{fa}_\mathcal{Y} : Q(B) \geq T_\#\underline{P}(B) \text{, } \forall B\in\mathcal{G}\}$ is a proper subset of $\mathcal{M}^\text{bc}(T_\#\underline{P})$, since $\mathcal{M}^\text{fa}(T_\#\underline{P})$ only considers the \textit{finitely additive} probabilities (and not all the bounded charges) that set-wise dominate $T_\#\underline{P}$.

Let now $(Q_\alpha)_{\alpha\in \mathcal{I}}$ be a net in $\mathcal{M}^\text{fa}(T_\#\underline{P})$ that weak$^\star$-converges to $Q\in\Delta^\text{fa}_\mathcal{Y}$. Here $\mathcal{I}$ is a generic index set. This means that $Q_\alpha(B) \rightarrow Q(B)$, for all $B \in\mathcal{G}$. That is, pick any $B \in\mathcal{G}$; then,
\begin{align}\label{eq-interm}
    \forall \epsilon>0 \text{, } \exists \tilde{\alpha}_\epsilon : \forall Q_\alpha \succeq Q_{\tilde{\alpha}_\epsilon} \text{, } |Q_\alpha(B) - Q(B)| < \epsilon.
\end{align}

Equation \eqref{eq-interm} implies that, for all $Q_\alpha \succeq Q_{\tilde{\alpha}_\epsilon}$, we have that $Q(B) > Q_\alpha(B)-\epsilon \geq T_\#\underline{P}(B) - \epsilon$. Letting $\epsilon \rightarrow 0$, this implies that $Q(B) \geq T_\#\underline{P}(B)$. But $B$ was chosen arbitrarily in $\mathcal{G}$, and so $Q\in\mathcal{M}^\text{fa}(T_\#\underline{P})$.
Hence, $\mathcal{M}^\text{fa}(T_\#\underline{P})$ is weak$^\star$ sequentially closed, and therefore it is  weak$^\star$ closed. Being a weak$^\star$ closed subset of a weak$^\star$ compact space, we can conclude that $\mathcal{M}^\text{fa}(T_\#\underline{P})$ is weak$^\star$ compact itself.

By \citet[Section 2.1.(viii)]{cerreia}, we know that a (set) function $\nu:\mathcal{G} \rightarrow [0,1]$ is a lower probability if and only if there exists a weak$^\star$-compact set $\mathcal{M} \subseteq \Delta^\text{fa}_\mathcal{Y}$ such that $\nu(B)=\min_{Q \in \mathcal{M}}Q(B)$, for all $B \in \mathcal{G}$. Letting $\mathcal{M} \equiv \mathcal{M}^\text{fa}(T_\#\underline{P})$ and $\nu\equiv T_\#\underline{P}$, we obtain that $T_\#\underline{P} (B) \leq \min_{Q \in \mathcal{M}^\text{fa}(T_\#\underline{P})} Q(B)$, for all $B\in\mathcal{G}$. Suppose now for the sake of contradiction that there is $\tilde B \in\mathcal{G}$ such that $T_\#\underline{P} (\tilde B) < \min_{Q \in \mathcal{M}^\text{fa}(T_\#\underline{P})} Q(\tilde B)$. This would imply that $\mathcal{M}^\text{fa}(T_\#\underline{P}) \supsetneq \mathcal{M}^\text{fa}(T_\#\underline{P})$, a contradiction. Hence, we can conclude that $T_\#\underline{P} (B) = \min_{Q \in \mathcal{M}^\text{fa}(T_\#\underline{P})} Q(B)$, for all $B\in\mathcal{G}$, thus completing the proof. \hfill \qedsymbol

\bibliographystyle{plainnat}
\bibliography{tropical}

\end{document}